\def\year{2019}\relax
\documentclass[letterpaper]{article} 
\usepackage{aaai19}  
\usepackage{times}  
\usepackage{helvet}  
\usepackage{courier}  
\usepackage{url}  
\usepackage{graphicx}  
\usepackage{times}
\usepackage{helvet}
\usepackage{courier}

\usepackage{multicol}
\usepackage{algorithm}
\usepackage{algpseudocode}
\usepackage{amsmath}
\usepackage{amssymb}
\usepackage{graphicx}
\usepackage{mathtools}
\usepackage{tikz}
\usepackage{times}
\usepackage{subfigure}
\usepackage{multirow}
\usepackage{graphics}
\usepackage{booktabs}
\usetikzlibrary{calc,shapes}

\usepackage{amsthm}

\DeclareMathOperator*{\argmax}{arg\,max}

\newtheorem{proposition}{Proposition}
\newtheorem{definition}{Definition}

\begin{document}
%

\title{Using Bi-Directional Information Exchange to Improve Decentralized Schedule-Driven Traffic Control}  

\author{Hsu-Chieh Hu \and Stephen F. Smith\\
 Carnegie Mellon University\\ Pittsburgh, PA, USA\\
hsuchieh@andrew.cmu.edu,sfs@cs.cmu.edu
 }

\maketitle
\begin{abstract}  

Recent work in decentralized, schedule-driven traffic control has demonstrated the ability to improve the efficiency of traffic flow in complex urban road networks. In this approach, a scheduling agent is associated with each intersection. Each agent senses the traffic approaching its intersection and in real-time constructs a schedule that minimizes the cumulative wait time of vehicles approaching the intersection over the current look-ahead horizon. In order to achieve network level coordination in a scalable manner, scheduling agents communicate only with their direct neighbors. Each time an agent generates a new intersection schedule it communicates its expected outflows to its downstream neighbors as a prediction of future demand and these outflows are appended to the downstream agent's locally perceived demand. In this paper, we extend this basic coordination algorithm to additionally incorporate the complementary flow of information reflective of an intersection's current congestion level to its upstream neighbors. We present an asynchronous decentralized algorithm for updating intersection schedules and congestion level estimates based on these bi-directional information flows. By relating this algorithm to the self-optimized decision making of the basic operation, we are able to approach network-wide optimality and reduce inefficiency due to strictly self-interested intersection control decisions.

\end{abstract}

\section{Introduction}
Over half of the world's population now lives in cities and global urbanization continues at a steady pace. As this trend continues, urban mobility is becoming an increasingly critical problem. In the US cities alone, the cost of congestion now exceeds \$160 Billion in lost time and fuel consumption, and is responsible for release of an additional $50$ Billion pounds of $CO_{2}$ into the atmosphere \cite{schrank20152015}. It is commonly recognized that better optimization of traffic signals could lead to substantial reduction of congestion and travel days, yet how to optimize a large transportation network in a responsive but scalable way remains a problem that continues to attract researchers from different fields. 

In urban environments, traffic signal control is still dominated by the use of fixed timing plans, which are based on average traffic conditions, and quickly become outdated as flow characteristics evolve over time. To improve matters, centralized approaches that adjust signal timing plan parameters (e.g., cycle time, green time split offset between consecutive intersections) according to actual sensed traffic data \cite{Robertson1991,Lowrie1992,heung2005coordinated,gettman2007data} have been proposed. However, these approaches are designed to accommodate continuous gradual change in traffic patterns (typically adjusting parameters after integrating information for several minutes), and are not responsive to real-time traffic events and disruptions. Alternatively, decentralized online panning approaches have been proposed \cite{sen1997controlled,gartner2002optimized,shelby2001design,cai2009adaptive,jonsson2011scaling}. These approaches solve the problem of scalability in principle, but have historically had difficulty computing plans in real-time with a sufficiently long horizon to achieve network-level coordination.

A recent development in decentralized online planning that overcomes this horizon problem and is capable of real-time responsiveness is schedule-driven traffic signal control \cite{Xie2012,xie2012schedule}. The key idea behind this approach is to formulate the intersection scheduling problem as a single machine scheduling problem, where input jobs are represented by a sequence of clusters consisting of spatially adjacent vehicles (i.e., approaching platoons, queues). This aggregate representation enables efficient generation of longer horizon plans that incorporate multi-hop traffic flow information (i.e., the traffic flow information across multiple intersections) and thus network-wide coordination is achieved through exchange of schedule information. In operation, an intersection scheduling agent is associated with each intersection. The goal of each scheduling agent is to allocate green time to different signal {\emph phases} over time, where a signal phase is a compatible traffic movement pattern (e.g., East-West traffic flow). Each agent generates in real time a phase schedule (or signal timing plan) for its intersection that minimizes the cumulative delay of approaching clusters. To collaborate with other agents, at each decision point each agent receives a projection of expected outflows from its upstream neighbors and considers these additional clusters when it generates a phase schedule. After starting to execute its schedule, the resulting flows are communicated to its downstream neighbors. Scalability is ensured by the fact that scheduling agents only communicate with their direct neighbors. However, outflow information (i.e., approaching upstream clusters) can propagate to non-local neighbors since the look-ahead horizon is extended and replanning occurs frequently. Results obtained in field experiments have shown significant reductions in travel times, wait times and number of stops, as well as in projected emissions. \cite{smith2013smart}.


One potential limitation of this approach, however, stems from its reliance on one-way flow of demand information from upstream intersections to downstream intersections. In cases where downstream intersections are in fact already congested, uninformed inflow of additional vehicles toward this intersection can further exacerbate delay and/or miss opportunities to better move cross street traffic. In effect, local intersection scheduling agents are optimizing in a self-interested manner, albeit with greater visibility of future demand, and this myopic perspective can compromise network-level performance.

In this paper, we consider the possibility of improving network-level performance by augmenting the information exchanged between neighboring intersections to include complementary downstream to upstream flow of congestion information. Our goal is to use shared cost (i.e., waiting time of vehicles) to improve the decision-making of each scheduling agent. The idea is to use estimates of downstream congestion cost to influence selfish upstream decisions and with this more global perspective, increase overall social welfare (network performance). Decision-making with shared cost (reward) or states has been proven to be an effective way to improve multi-agent problem solving performance in other settings \cite{huang2006distributed,Yang2010}.

We propose an expanded bi-directional information exchange algorithm between intersections that combines forward communication of projected vehicle outflows to downstream intersections with backward communication of the estimated delay for each vehicle to upstream intersections as a prediction of next-hop costs. This additional information is incorporated by redefining the local intersection scheduling objective to include these costs. In situations where traffic is light, the feedback delay will be small and local intersection scheduling will proceed as before. However as the network becomes saturated and the cumulative delay of downstream neighbors becomes larger, the feedback cost will reflect this and lessen the number of vehicles that are sent downstream in this direction. To ensure scalability, messages continue to be exchanged only between direct neighbors and the asynchronous nature of local intersection scheduling is preserved.

The remainder of the paper is organized as follows. We first introduce the related work. Next, the problem formulation and the detailed algorithm necessary to achieve better coordination are presented. Then, an empirical analysis of the proposed approach is shown. Finally, the conclusions are drawn.

\section{Related Work}
A general review of all past intersection control schemes is beyond the scope of this paper; we refer readers to \cite{shelby2001design} and \cite{stevanovic2010adaptive} for more comprehensive overviews. As mentioned earlier, since decentralized control schemes have been explored as a means for increasing number of signals and detectors while maintaining real-time responsiveness, we briefly summarize several agent-based approaches that optimize traffic flow in a decentralized way. 

It is well established that agent-based approaches are well suited to the decentralized traffic management problem, given newly developed sensing technologies, historical temporal data, and the frequent and flexible interaction between the agents and their environment \cite{dresner2008multiagent,bazzan2014review}. A common approach related to control of traffic signals is to let multiple agents learn a policy for mapping states to actions by monitoring traffic flow and selecting actions. A Markov Decision Process (MDP) is a popular means for model this problem \cite{camponogara2003distributed}. Since the space of state-action pairs grows exponentially and also depends on discretization of states and number of intersections, scalability to actual traffic conditions may be problematic. Hence, instead of solving a large MDP, use of an independent learner can relax this problem. In \cite{da2006dealing}, model-based reinforcement learning is proposed to deal with dynamic non-stationary traffic flow, although it still lacks consideration of joint states and joint actions. Moreover, learning is a time-consuming task and imposes an overhead on real-time control. It is challenging to learn a policy to deal with all kinds of traffic conditions in real-time. Techniques of evolutionary game theory in which agents perform experimentation and receive a reward that depends on the neighbors is used in  \cite{bazzan2005distributed}. This approach becomes time-consuming when many different options of coordination are possible.

To deal with computational complexity of joint optimization, a recent trend is to let agents learn independently but allow them to interact with each others and combine their policies or plans. This provides a new trade-off between total centralization and total independence.  Exchange of information between a group of agents may increase accuracy and learning speed at the expense of communication \cite{nunes2004learning}. The work in \cite{kuyer2008multiagent} also focuses on exchanging information to benefit reinforcement learning and explicit coordination among agents through a coordination graph. However, this approach leads to an increase in complexity as the graph becomes larger. 

In the field of planning, exchanging information to extend the horizon is considered in  \cite{sen1997controlled,gartner2002optimized,Xie2012,xie2012schedule} as a way to accommodate non-local information. In addition, communication with more accurate information has been shown to be effective for multi-agent online planning  \cite{wu2009multi}.

\section{Problem Formulation}
We start by introducing scheduling model and notations of schedule-driven traffic control and then define the network-level online planning problem that provides the objective pursued in this paper for better network-level coordination. 
\subsection{Schedule-Driven Traffic Control}
As indicated above, the key to the single machine scheduling problem formulation of the schedule-driven approach of \cite{Xie2012,xie2012schedule} is an aggregate representation of traffic flows as sequences of clusters $c$ over the planning (or prediction) horizon. Each \textit{cluster} $c$ is defined as $(|c|, arr, dep)$, where  $|c|$, $arr$ and $dep$ designate number of vehicles, arrival time of the cluster at the intersection, and its departure time respectively. Vehicles entering an intersection are clustered together if they are traveling within a pre-specified interval of one another. The clusters become the jobs that must be sequenced through the intersection (the single machine). Once a vehicle moves through the intersection, it is sensed and grouped into a new cluster by the downstream intersection.The sequences of clusters provide short-term variability of traffic flows at each intersection and preserve the non-uniform nature of real-time flows. Specifically, the \textit{road cluster sequence} $C_{R,m}$ is a sequence of $(|c|, arr, dep)$ triples reflecting each approaching or queued vehicle on entry road segment $m$ that is ordered by increasing $arr$. Since it is possible for more than one entry road to share the intersection in a given \textit{phase} (a phase is a compatible traffic movement pattern, e.g., East-West traffic flow), the \textit{input cluster sequence} $C$ can be obtained through combining the road cluster sequences $C_{R,m}$ that can proceed concurrently through the intersection. The travel time on entry road $m$ defines a finite horizon ($H_m$), and the prediction horizon $H$ is the maximum over all roads.

Each cluster in an approaching road segment is viewed as a non-divisible job and a forward-recursion dynamic programming search is executed in a rolling horizon fashion to continually generate a phase schedule that minimizes the cumulative delay of all clusters. A new phase schedule is generated once a second for purposes of reducing uncertainty associated with clusters and queues. The process constructs an optimal sequence of clusters that maintains the ordering of clusters along each road segment, and each time a phase change is implied by the sequence, then a delay corresponding to the intersection's yellow/all-red changeover time constraints is inserted.  If the resulting schedule is found to violate the maximum green time constraints for any phase (introduced to ensure fairness), then the first offending cluster in the schedule is split, and the problem is re-solved.

Formally, the resulting \textit{control flow} can be represented as a tuple $(S, C_{CF})$ shown in Figure~\ref{cf}, where $S$ is a sequence of phase indices, i.e., $(s_1, \cdots, s_{|S|})$, $C_{CF}$ contains the sequence of clusters $(c_1, \cdots, c_{|S|})$ and the corresponding starting time after being scheduled. More precisely, the delay that each cluster contributes to the cumulative delay $\sum_{k = 1}^{|S|} d(c_{k})$ is defined as 
\begin{equation}
d(c_{k}) = |c_{k}| \cdot (ast - arr(c_{k})),
\label{orig_delay}
\end{equation} 
where $ast$ is the actual start time that the vehicle is allowed to pass through, which is determined by the optimization process. The optimal sequence (schedule) $C_{CF}^*$ is the one that incurs minimal delay for all vehicles.

 \begin{figure}[!htbp]
\centering
\includegraphics[scale = 0.35]{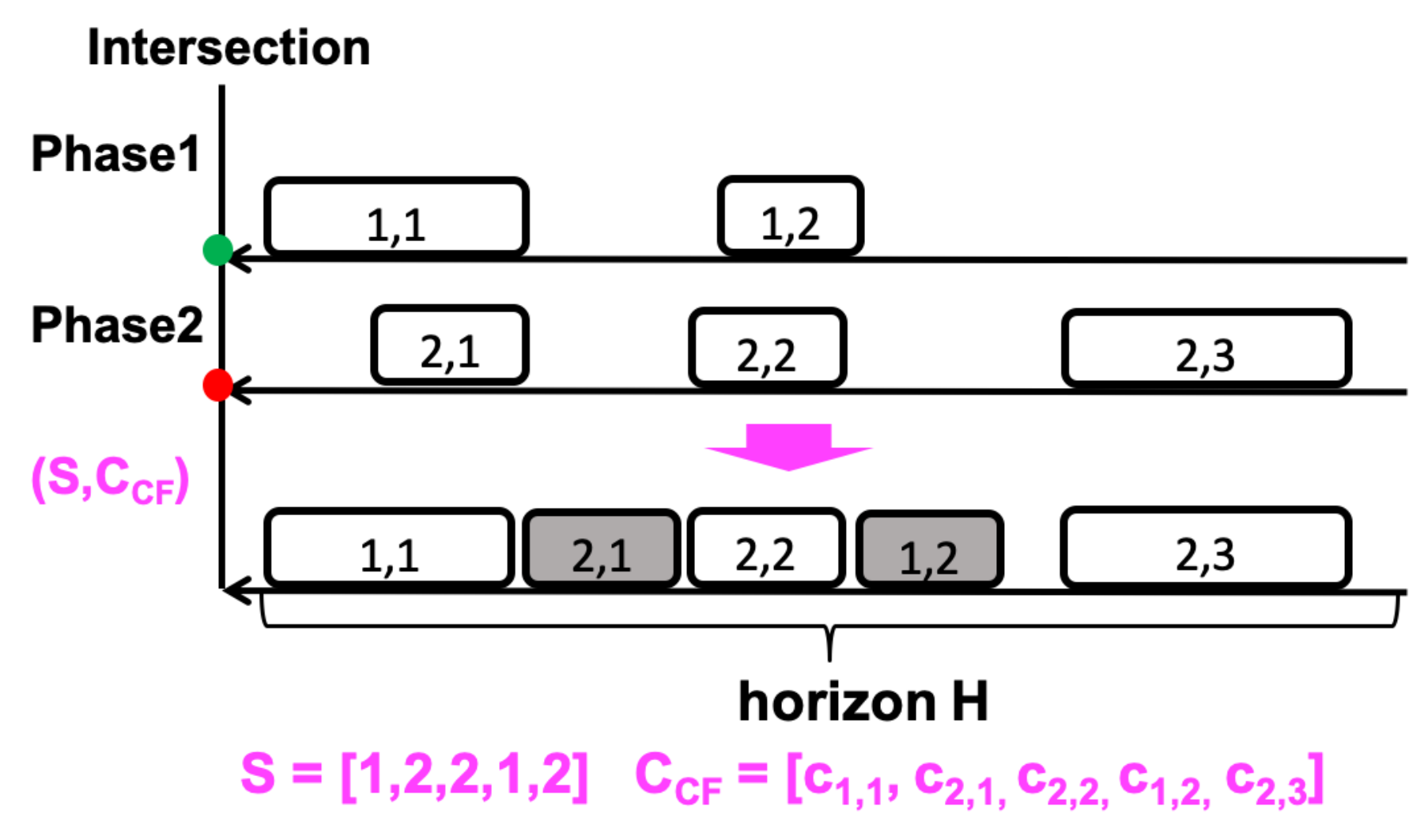}
\caption{The resulting control flow $(S, C_{CF})$ calculated by scheduling agents: each block represents a vehicular cluster of \textit{input cluster sequence} $C$, which combines the road cluster sequences $C_{R,m}$. For instance, $(2,1)$ represents the first cluster at phase $2$. The shaded blocks of $C_{CF}$ represent the delayed clusters.}
\label{cf}
\end{figure}

To collaborate with neighbor intersections, each intersection receives a projection of expected outflows from its upstream neighbors and plugs it into its local computation. After starting to execute its schedule, the resulting flows are communicated to its downstream neighbors. Since a vehicle may enter into/leave from intersection via different road segments, the clusters that are propagated to neighbors over extended look-ahead horizon $H$ are split and weighted by turning movement proportion. Thus, the weight $|c|$ of the non-local cluster will be a fractional number to reflect the uncertainty of movement. The turning movement proportion data is estimated by taking average of traffic flow rates for different phases. All approaching vehicles are sensed through the intersection's lane detectors.

\subsection{Network-Level Online Planning Problem}
As mentioned earlier,  our hypothesis is that the effectiveness of this schedule-driven process is restricted by the fact that as each scheduling agent aims to optimize its own cumulative delay without regard to the cost it imposes on others. To formulate the problem, we model a transportation network by a graph $G = \{V,E\}$, where vertex $v \in V$ is an intersection and $e \in E$ is the road segment connecting the intersections. Since schedule-driven traffic control is an online planning approach,  overall performance can be formulated as the sum of the following coupled objective that is continually re-optimized at each replanning time $t$ for the current prediction horizon $H$:
\begin{equation}
\min_{\{C_{CF,i}(t), i \in V\}}\sum_{i \in V} f_i(C_i(t), C_{-i}(t)), 
\label{overall_perf}
\end{equation}
where $C_i(t)$ and  $C_{-i}(t)$ are the local cluster sequences of approaching vehicles at intersection $i$ and intersections other than $i$ respectively, $C_{CF,i}$ determines $ast$ (actual start times) of  the input clusters and thus how local clusters propagate to downstream, and $f_i(C_i(t), C_{-i}(t)) = \sum_{k = 1}^{|S|} d(c_{k})$ is the cumulative delay of intersection $i \in V$ given the schedules of all intersections except $i$,
\begin{equation}
C_{-i}(t)=(C_1(t), \cdots, C_{i-1}(t), C_{i+1}(t),\cdots,C_{|V|}(t)).
\end{equation}
Note that the cumulative delay at intersection $i$ is not merely determined by the local clusters $C_i$ but also the propagated $C_{-i}$ (i.e., outflow information) sent by other intersections within $H$.  However, due to the combinatorial nature of the scheduling problem, solving this network-wide scheduling problem exactly is computationally intractable, especially if the horizon $H$ is extended sufficiently by including flow information from multiple intersections and there are many intersections to coordinate. Here, we consider a formulation that can be solved in a decentralized way with only communication of direct neighbors and their local clusters.

\begin{definition}[Overall Performance with a Finite Horizon]
Assuming that the indirect impact of an intersection schedule to intersections that are two or more hops away is negligible through a finite horizon $H$, we have the following optimization problem:
\begin{equation}
\min_{\{C_{CF,i}(t), i \in V\}}\sum_{i \in V} f_i(C_i(t), C_{\mathcal{N}_i}(t)), 
\label{relax_perf}
\end{equation}
\end{definition}
where $\mathcal{N}_i$ is the set of direct neighbors of intersection $i$ and $C_{\mathcal{N}_i}(t)$ are the scheduled clusters sent by these neighbors. If a longer look-ahead horizon is allowed, more intersections can be inserted into the set $\mathcal{N}_i$. Under light traffic conditions, (\ref{relax_perf}) hints that a good approximate solution is one where each agent optimizes its local objective greedily, since less traffic is created toward others. Considering the performance of this schedule-driven process in a network that is experiencing high congestion, however, the coupling of traffic across intersections is dominant in this delay computation. The remedy is to bias the scheduling search more toward reducing joint delay across neighboring intersections as the level of local congestion increases. 

\section{Bi-Directional Information Exchange}
 In this section, we introduce a distributed algorithm for calculating the schedule of a given intersection, so that its results are better coordinated with the schedules of neighboring intersections. In brief, we propose an asynchronous decentralized algorithm in which agents generate a harmonized joint timing plan through reciprocal exchange of downstream congestion cost information in addition to exchanged schedule outflow information. With this extra information, the intractable network-level optimization problem can be approximated by locally planning according to a modified objective that incorporates this information.

\subsection{Congestion Feedback}
As mentioned earlier, the control efficiency of a signalized network not only depends on how a single intersection allocates green time efficiently but also is affected by how much traffic it imposes on others. According to schedule-driven traffic control, the agent is able to make the optimal decision based on the observed approaching vehicles within a finite horizon. To push the boundary of performance further, we incorporate next-hop delay into the optimization.

To estimate the next-hop delay, we need to divide the control flow $C_i(t)$ according to the corresponding phases. For each intersection with a set of entry and exit roads, traffic on a given exit road is sent to the downstream neighbor that corresponds to that traffic phase. The traffic light cycles through a fixed sequence of phases $P$, and each phase $p \in P$ governs the right of way for a set of compatible movements from entry to exit roads. Therefore, the sequence $C_{CF,i}(t)$ at intersection $i$ can be decomposed into $|P|$ sub-sequences $(C_{1,i}(t),\cdots,C_{|P|,i}(t))$, where $C_{p,i}(t)$ contains clusters $(c_{p,1},\cdots, c_{p,|S_p|})$ with the right of way during phase $p$ and $S_p$ designates indices of clusters.

To illustrate the idea of incorporating next-hop delay into computation, the overall performance (\ref{relax_perf}) is rewritten in terms of intersection $i$ as
\begin{equation}
f_i(C_i(t), C_{\mathcal{N}_i}(t)) + \sum_{j \neq i} f_j(C_j(t), C_{\mathcal{N}_j}(t)).
\label{price_obj}
\end{equation}

Specifically, (\ref{price_obj}) is viewed as an approximation of the global objective (\ref{overall_perf}) for the intersection $i$, so that minimizing (\ref{price_obj}) biases the local decision toward better social welfare. If we assume that outflow information $C_{\mathcal{N}_i}(t-1)$ and others' schedule $C_{CF,j}(t-1), j \neq i$ are received at time $t$ by intersection $i$, i.e., each agent is an information taker and ignores any immediate influence it has on this information, each intersection $i$ solves the following problem to approach optimal social welfare:
\begin{align}
\min_{C_{CF,i}(t)} \quad&f_i(C_i(t), C_{\mathcal{N}_i}(t-1)) \nonumber\\
+ &\sum_{j \neq i} f_j(C_j(t-1), \{C_{\mathcal{N}_j\backslash i}(t-1),C_i(t)\}),
\label{local_price_obj}
\end{align}
where $\mathcal{N}_j\backslash i$ denotes neighbor intersections of $j$ except $i$. The control flow $C_{CF,i}(t)$ decides the $ast$ (actual start time) of $(C_i(t), C_{\mathcal{N}_i}(t-1))$ and thus the $arr$ (arrival time) of $C_i(t)$ at downstream intersections, where $C_i(t)$ is used to represent both the input cluster sequence at intersection $i$ and the outflow information received by other intersections.  The problem can be further simplified by removing irrelevant terms to $C_i(t)$.
\begin{proposition}[Biased Local Objective]
 Since the $C_i(t)$ only exists in the local objective of $\mathcal{N}_i$, minimizing (\ref{local_price_obj}) at time $t$ is equivalent to solving 
\begin{align}
\min_{C_{CF,i}(t)}\quad &f_i(C_i(t), C_{\mathcal{N}_i}(t-1)) \nonumber\\ 
+ &\sum_{j \in \mathcal{N}_i} f_j(C_j(t-1), \{C_{\mathcal{N}_j\backslash i}(t-1),C_i(t)\}).
\label{final_price_obj}
\end{align}
\end{proposition}
\begin{proof}
By Equation (\ref{relax_perf}), $C_i(t)$ only exists in the $f_i$ and $f_j, j \in \mathcal{N}_i$.
\end{proof}
From (\ref{final_price_obj}), the second term considers the number of vehicles sent to neighbors according to $C_i(t)$. More specifically, the possible delay of sent vehicles at intersections other than $i$ should be taken into account if the scheduling agent of intersection $i$ attempts to compute a schedule $C_{CF, i}(t)$ toward social welfare.  For instance, if intersection $j$ is the next-hop of $c_{p,k}$ in the direction of phase $p$, only $C_{p,i}(t)$ can contribute to the term $f_j(C_j(t), C_{\mathcal{N}_j}(t))$ in (\ref{price_obj}). Basically, solving (\ref{final_price_obj}) improves overall delay performance compared to baseline schedule-driven approach that solves local objective individually, as shown in the Proposition \ref{optimal_proof}.
\begin{proposition}[Reduce two hop delay]
To any vehicles, the cumulative delay of passing through two consecutive intersections is improved by solving (\ref{final_price_obj}) compared to minimizing local objective $f_i(C_i(t), C_{\mathcal{N}_i}(t-1))$ independently, i.e., baseline approach, given previous neighbor information.
\label{optimal_proof}
\end{proposition}
\begin{proof}
From (\ref{final_price_obj}), the summation of cumulative delay for all vehicles $v\in C_i(t), C_{\mathcal{N}_i}(t-1)$ to pass through intersection $i$ and the corresponding $\mathcal{N}_i$ is minimum.
\end{proof}
However, computing actual next-hop delay is unpractical due to the nature of the combinatorial problem. Intuitively, the contribution of $C_{p,i}(t)$ can be estimated by the average delay of sent vehicles in phase $p$. We introduce a feedback, called \textit{congestion feedback} and denoted by $\hat{d}[C_{p,j}(t-1)]$, to quantify this contribution. Through the cluster representation of schedule-driven traffic control, we have an intuitive way to estimate $\hat{d}[C_{p,j}(t-1)]$
\begin{definition}[Congestion Feedback]
Intersection $j$ computes its average delay of phase $p$ and sends to its neighbor corresponding to phase $p$. Then, we can define congestion feedback sent from intersection $j$ by
\begin{equation}
\hat{d}[C_{p,j}(t-1)] =\frac{ \sum_{c_{p,k}\in C_{p,j}(t-1)} d(c_{p,k})}{\sum_{c_{p,k}\in C_{p,j}(t-1)} |c_{p,k}|}.
\label{congestion_feedback}
\end{equation}
\end{definition}
The numerator is the total cumulative delay in phase $p$, and the denominator is the total number of vehicles in that phase. $\hat{d}[C_{p,j}(t-1)]$ is the estimated next-hop delay of $c_{p,k}$ for each vehicle at intersection $j$ according to control flow $C_{p,j}(t-1)$ at the previous time step.  Using the notion of congestion feedback, we can propose a new version of delay for each cluster at the intersection $i$ that incorporates the cost it imposes on others:
\begin{definition}[Augmented Delay]
The next hop of $c_{p,k}$ is intersection $j$. Then, its two hop delay can be represented as
\begin{equation}
d(c_{p,k}) = |c_{p,k}| \cdot \big[(ast - arr(c_{p,k})) + \hat{d}[C_{p,j}(t-1)]\big],
\label{aug}
\end{equation}
where  $c_{p,k}\in C_{p,i}(t)$.
\end{definition}
We generate schedule that minimizes summation of the augmented delay (\ref{aug}) (i.e., cumulative delay) other than the original delay (\ref{orig_delay}).


(\ref{congestion_feedback}) can serve as an accurate predictor of next-hop delay since it is based on replanning at the previous time step. If the granularity of the replanning is every second or even a smaller time unit, the traffic condition should not shift too drastically.   By introducing (\ref{congestion_feedback}) in each phase, the number of vehicles corresponding to a specific phase can be adjusted within the finite horizon $H$. Larger next-hop delay implies that sending more vehicles to a congested intersection would increase overall cumulative delay with a higher probability. The reduction of overall performance could be dominant compared to the increment of local objective. If the next-hop delay is small, which means that the traffic of neighbors are light, the schedule is similar to the original unbiased one.  The integration of next-hop delay motivates the decentralized algorithm described below. 


\subsection{Decentralized Congestion Compensation}
In this section, we present how to combine forward communication of projected vehicle outflows with backward communication of congestion feedback to coordinate a signalized network. The backward congestion feedback reflects compensation for imposing traffic on downstream traffic. At the beginning, each intersection announces its $|P|$ congestion cost measures to its upstream neighbors corresponding to different phases, and each neighbor factors the cost it receives into the computation of its schedule as described in Figure~\ref{cf_graph}.  After collecting all bi-directional information, intersection $i$ computes schedule according to
\begin{equation}
C_{CF,i}(t) = \argmax_{\hat{C}_{CF,i} = (\hat{C}_{1,i},\cdots, \hat{C}_{|P|,i})}\sum_{p = 1}^{|P|}\sum_{c_{p,k}\in \hat{C}_{p,i}} d(c_{p,k}).
\label{up_c}
\end{equation}
Each intersection then updates its congestion feedback according to  (\ref{congestion_feedback}). In this model, the cost and schedule are asynchronously updated. The decentralized congestion compensation (DCC) algorithm is given as follows 

\scalebox{0.9}{
\begin{minipage}{\columnwidth}
\begin{algorithm}[H]
\floatname{algorithm}{The DCC Algorithm}
\renewcommand{\thealgorithm}{}
\caption{Steps defining how intersection $i$ communicates to its downstream neighbors to achieve "social welfare" of the network}
\label{protocol1}
\begin{algorithmic}[1]
\State \textbf{Initialization}: For intersection $i \in V$ generate a initial schedule $C_i(0)$ and set the congestion feedback to $0$. 
\State   \textbf{Receive congestion feedback and outflow information}: At each time $t$, intersection $i$ receives congestion feedback from downstream of $j \in \mathcal{N}_i$, which is
$\hat{d}[C_{p,j}(t-1)] $, and schedule (outflow information) from upstream of $j \in \mathcal{N}_i$.
\State \textbf{Forward-recursion  dynamic programming search}: Intersection $i$ computes its schedule $C_i(t)$ according to equation (\ref{up_c}).
\State  \textbf{Feedback congestion feedback and outflow information}: According to equation (\ref{congestion_feedback}), intersection $i$ calculates $\hat{d}[C_{p,i}(t)]$  and schedule and shares them with upstream and downstream neighbors. Return to step 2.
\end{algorithmic}
\label{protocol}
\end{algorithm} 

\end{minipage}
}\\

In the DCC algorithm, it can be seen that to implement those updates, each intersection $i$ needs to know only: 1) its traffic flow, i.e., all approaching clusters within the horizon $H$ and 2) the neighbor congestion feedback. Although congestion feedback is computed based on neighbor's previous schedule, we assume that online planning with replanning frequently (e.g., every second) can resolve this freshness problem and generate an accurate prediction of future traffic.


\subsection{Outdated Information Prevention}
Since (\ref{congestion_feedback}) and (\ref{aug}) of the previous section are in a recursive form, the information propagated from distant intersections could be embedded in the congestion feedback. Although such multi-hop information could reflect the traffic conditions of other intersections in a certain sense, this information may be outdated. For instance, the congestion information embedded in the feedback may imply a clogged neighbor 5 minutes ago, but the traffic is already cleared when the information arrives at the current local intersection. 

Other than (\ref{aug}), there may be other ways to combine these delay quantity, e.g., weighted average. However, numerical results show that the performance is similar if intersections only share their actual local delay to neighbor intersections rather than a composite multi-hop delay by different combining methods. It can be seen that there is a trade-off between "freshness" and propagation distance.  Uncertainty marginalizes the advantage of including more information in the optimization.

In order to reduce complexity of real-time system and avoid the aged information problem, the local delay is plugged into (\ref{congestion_feedback}) instead. We maintain two tables recording augmented and non-augmented (local) delay information respectively when applying the dynamic programming search. The search is done with the first table, which records those transitions based on the augmented delay (\ref{aug}). The schedule is generated by the first table. On the other hand, the second table maintains the non-augmented delay $d_{local}(c_{p,k}) = |c_{p,k}| \cdot \big(ast - arr(c_{p,k})\big)$ when the search is running and uses it to calculate congestion feedback

\begin{equation}
\hat{d}[C_{p,j}(t-1)] =\frac{ \sum_{c_{p,k}\in C_{p,j}(t-1)} d_{local}(c_{p,k})}{\sum_{c_{p,k}\in C_{p,j}(t-1)} |c_{p,k}|}.
\label{local_congestion_feedback}
\end{equation}

 By applying these two tables, intersection can determine the congestion feedback from second table and share them with neighbors, so that it can prevent outdated information from flowing within the network.  This approach to managing outdated information is applied by default in our evaluation.
 
 \begin{figure}[!htbp]
\centering
\includegraphics[scale = 0.3]{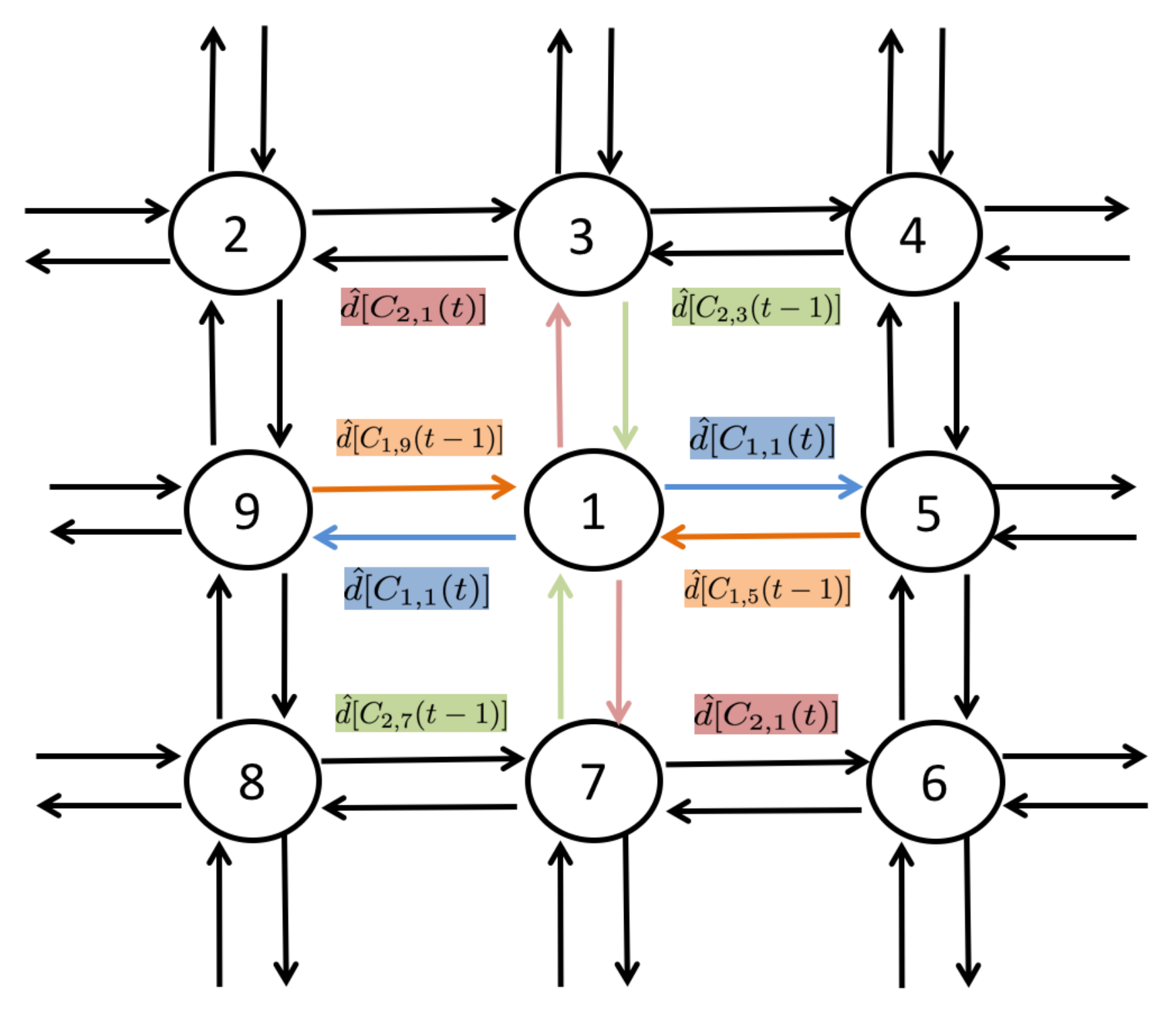}
\caption{Exchange congestion feedback with neighbor intersections: intersection $3$ and $7$ belong to phase $2$ of intersection $1$, and intersection $5$ and $9$ belong to phase $1$. Each intersection sends and receives congestion feedback to and from neighbor intersections.}
\label{cf_graph}
\end{figure}

\subsection{Bottleneck Prevention}
Considering the case where the congestion level (or loading) of all downstream neighbors is lower than the local loading, the primary task of the local agent should be evacuating the approaching vehicles as soon as possible. Otherwise the local traffic could possibly reach the physical road capacity. Alternatively, if the local agent has lower congestion than one or more of its downstream neighbors, then its priority should be to slow down traffic evacuation in the appropriate direction (since the traffic will be delayed anyway). To deal with this case, we design a \textit{bottleneck criterion} (BC) based on the newly computed schedule and the latest received congestion feedback:
\begin{definition}[Bottleneck Criterion]
The intersection $i$ satisfying 
\begin{equation}
\hat{d}[C_i(t-1)] \cdot w_i + \epsilon \geq \hat{d}[C_j(t-1)] \cdot w_j ,\quad j\in \mathcal{N}_i
\end{equation}
is viewed as a bottleneck and may optimize cumulative non-augmented delay instead.
\end{definition}
$C_i(t-1)$ and $C_j(t-1)$ are all input clusters at intersection $i$ and $j$. $\epsilon$ is a parameter to make sure that local loading is sufficiently larger than that of downstream neighbors and $w_i$ is a weight being proportional to corresponding road capacity. Note that the congestion feedback used in this criterion is computed by $C_i(t-1)$ and $C_j(t-1)$, and reflects the aggregate traffic condition of all phases. If the criterion is satisfied, then intersection $i$ uses non-augmented delay to compute its schedule. In essence, it means that the agent returns to self-interested mode.

\subsection{Turning Movement Proportion}
Considering turning proportions at each intersection is crucial for improving performance of adaptive traffic signal systems. In the baseline schedule-driven approach, the turning movement proportion is estimated by taking moving averages of traffic flow rate for different phases respectively. The lane detectors detect the numbers of turning vehicles, compute the moving average and then normalize these flow rates. After getting these proportions, the scheduled flow is able to reflect the realistic traffic flow by proportioning the add-on flow and evacuated flow. For a grid-like network, the congestion feedback from three input links (e.g., east, north, and west) of the downstream intersections should be multiplied by the corresponding turning proportions and summed up together to obtain the effective congestion feedback to local input link (north). If $c_{p,k}$ is from intersection $u$ to intersection $i$, the effective congestion feedback can be defined as follows,
\begin{definition} [Effective Congestion Feedback]
Let $c_{p,k}$ be the $k$th cluster in the $C_{p,i}(t)$ of intersection $i$ from intersection $u$. Then
\begin{equation}
\tilde{d}(c_{p,k}) = \sum_{j\in \mathcal{N}_i\backslash u} \zeta_{u,j} \cdot \hat{d}[C_{P(i,j),j}(t-1)] 
\end{equation}
is the effective congestion feedback, where $\zeta_{u,j}$ is the turning proportions of input and output links between intersection $u$ and $j$ and $P(i,j)$ is corresponding phase of intersection $j$ to intersection $i$.
\end{definition}
The corresponding augmented delay is 
\begin{equation}
d(c_{p,k}) = |c_{p,k}| \cdot \big[(ast - arr(c_{p,k})) + \tilde{d}(c_{p,k})\big],
\label{turn_aug}
\end{equation}
In the following experimental evaluation, the effective congestion feedback is applied by default to deal with the uncertainty of turning movement.

\section{Experimental Evaluation}
In this section, we compare DCC algorithm to two other real-time traffic control methods. First, we take the performance of the original schedule-driven traffic control system \cite{Xie2012,xie2012schedule} as our baseline system. Second, we compare to a variant of cycle-based adaptive control that optimizes cycle time, phase split and timing offset of successive signals every cycle. The basic concept of cycle-based adaptive control is to calculate cycle time based on estimation of saturation flow rate \cite{webster1958traffic} and allocate green time according to flow ratio on each phase. A well known example of this type of adaptive control scheme is the SCATS system \cite{daizong2003comparative,wongpiromsarn2012distributed}. \footnote{Note also that previous research with the baseline schedule-driven approach has shown its comparative advantage over other online planning approaches  \cite{Xie2012,xie2012schedule}.}

To evaluate our approach, we simulate performance on a two-intersection model and a real world network. The two-intersection model is for studying how different traffic patterns (i.e., symmetric or asymmetric) affect performance. The real world network is for evaluating the performance of DCC in a larger complex real network.  The simulation model was developed in VISSIM, a commercial microscopic traffic simulation software package. We assume that each vehicle has its own route as it passes through the network and measure how long a vehicle must wait for its turn to pass through all intersections along its route (the delay). Tested traffic volume is averaged over sources at network boundaries. To assess the performance boost provided by the DCC, we measure the average waiting time of all vehicles over ten runs. All simulations run for $3.5$ hour of simulated time. Results for a given experiment are averaged across $10$ simulation runs with different random seeds.

\subsection{Two-Intersection Model}
We consider a simple model with two connected $2$-way intersections that have multiple lanes for each direction as controlled experiments. By changing external flow rates, two types of traffic scenarios are tested: 1) symmetric traffic  and 2) asymmetric traffic. In this simple model, there is only one connecting road segment. The maximum traffic volume is set to $2800$ cars/hour due to speed limit and road capacity. 

Figure~\ref{bal} (a) shows that DCC algorithm is able to handle high volume better than the benchmark. When the traffic volume increases, sending too much traffic to the connecting road segment will deteriorate the traffic condition.  If the congestion feedback is large, dynamic programming search will decrease the number of vehicles sent to the connecting road and reduce the cumulative delay. Under lighter traffic situations, the performance of both DCC and baseline are comparable since small congestion feedback is not strong enough to bias the schedule. Furthermore, DCC with the bottleneck criterion avoids the situation that both intersections produce suboptimal schedules to local traffic under light traffic conditions and thus achieve better performance.

In Figure~\ref{bal} (b), we can observe that DCC is especially useful when the traffic is asymmetric.  In this controlled experiment, we fix the traffic volume of one intersection (right intersection) and increase the volume of the other one (left intersection) step by step (from $0\%$ to $40\%$). DCC provides $20\%$ and $35\%$ delay reduction at most compared with the benchmark and the cycle-based adaptive control scheme. When traffic becomes heavier in one of those intersections, congestion feedback coordinates one intersection to send more vehicles and the other one to send less along the connecting road. Note also that performance of two DCC variants are comparable. 

It is interesting to note that the performance gain for asymmetric traffic is greater than the symmetric traffic. If the traffic pattern is symmetric, it becomes more difficult to differentiate the loading between neighbors. To improve the performance further, detailed schedule information may be required in addition to congestion feedback.


\begin{figure}[!htbp]
\centering
\subfigure[Symmetric traffic]{\includegraphics[width=41.5mm, height=33mm]{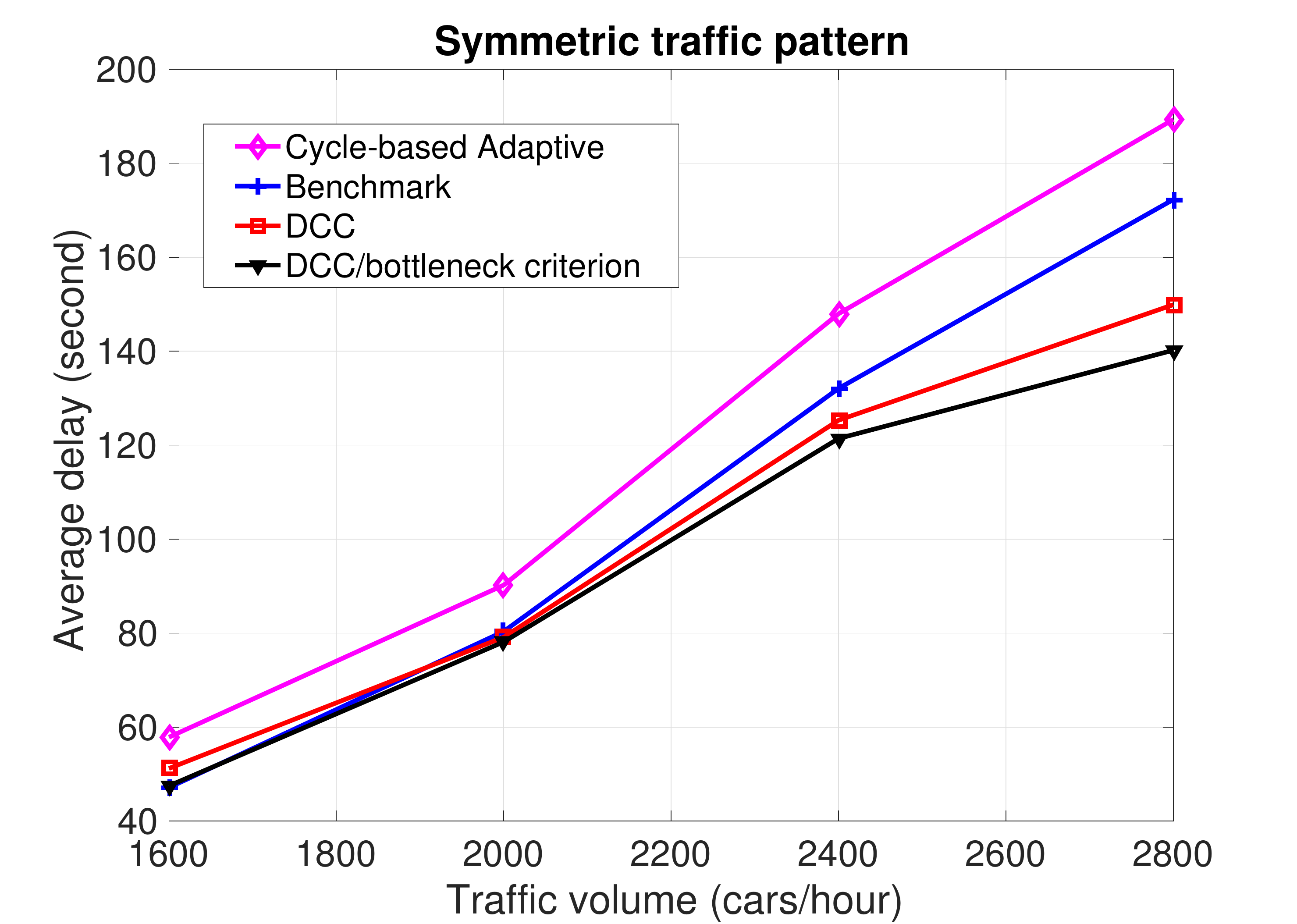}} 
\subfigure[Asymmetric traffic]{\includegraphics[width=41.5mm, height=33mm]{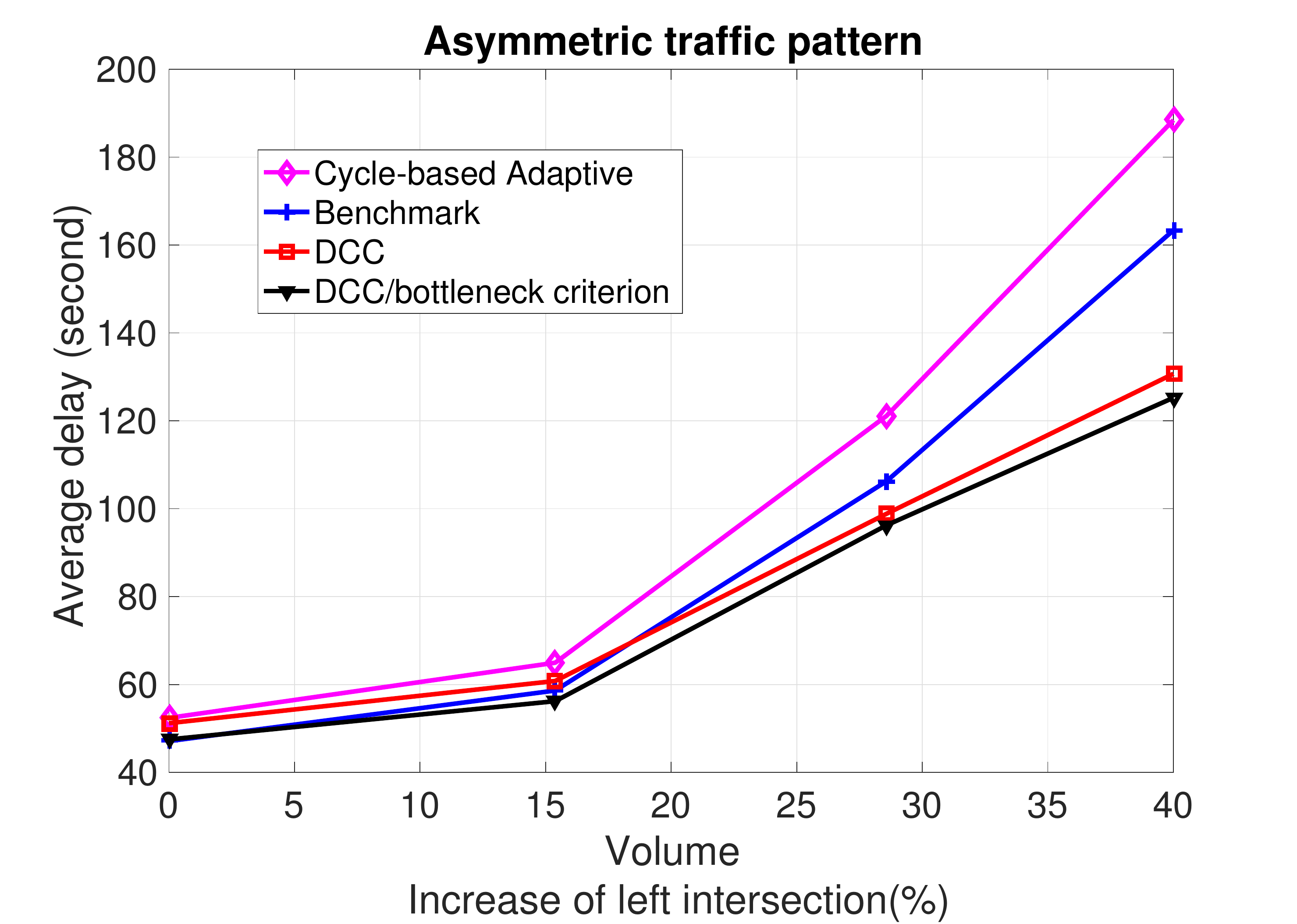}}
\caption{The delay of symmetric and asymmetric traffic pattern: (a) traffic of both intersections are increasing. (b) only traffic of left intersection is increasing}
\label{bal}
\end{figure}

%
%

\subsection{Urban Network Model}
The network model is based on the Baum-Centre neighborhood of Pittsburgh, Pennsylvania as shown in Figure~\ref{surtracplusmap}. The network consists of $24$ intersections that are mainly 2-phased. It can be seen as a two-way grid network. All simulation runs were carried out according to a realistic traffic pattern from late afternoon through "PM rush" (4-6 PM). The traffic pattern ramps up volumes over the simulation interval as follows:  (0-30mins: $472$ cars/hour, 30min-1hour: $708$ cars/hour, 1hour-2hours: $1056$ cars/hour ). This simulation model presents a complex practical application to verify the effectiveness of the proposed approach.

\begin{figure}[!htbp]
\centering
\includegraphics[scale = 0.20]{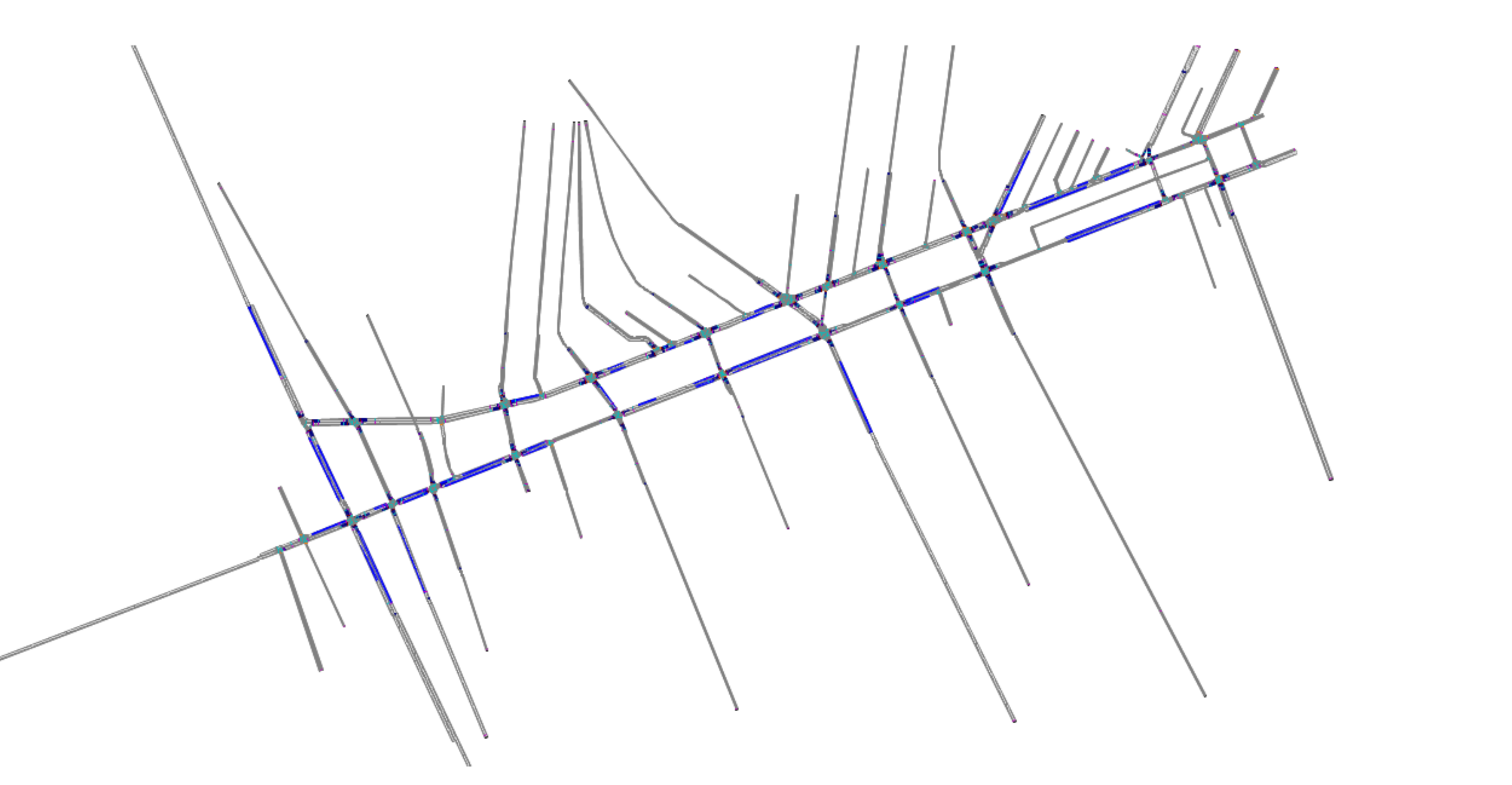}
\caption{Map of the 24 intersections in the Baum-Centre neighborhood of Pittsburgh, Pennsylvania}
\label{surtracplusmap}
\end{figure}

\begin{table}[!htbp]
\centering
  \scalebox{0.6}{
  \begin{tabular}{*{4}{c}}
   \toprule
      \multirow{2}{*}{}& \multicolumn{2}{c}{ Average Delay (second)}  \\
    
   \cmidrule(l){2-3}\cmidrule(l){4-4}    
   & mean  & std. & stop no. \\
    \midrule
 Benchmark & 147.00 & 177.94  & 8.27\\  
 Cycle-based Adaptive & 169.23 & 265.91  & 10.81\\  
DCC & 121.56 & 100.65 & 5.33\\ 
DCC w/ BC& 116.01 & 93.22  &5.32\\

    \bottomrule
  \end{tabular}
  }
   \caption{Summary of Baum Centre Model Results}
  \label{resultstable1}
\end{table}

Table~\ref{resultstable1} shows the results of DCC under PM rush, compared to cycle-based adaptive control approach and the baseline schedule-driven approach. In addition to DCC, we also compare DCC with the additional criterion (i.e., bottleneck criterion) for identifying bottleneck intersections.  As can be seen, delay is reduced by $17.7\%$ and $28.4\%$, compared to the schedule-driven and adaptive control approaches respectively. The use of congestion feedback reduces delay by coordinating intersections through exchanging schedule information. Furthermore, the use of the congestion feedback is beneficial for clearing queues of waiting vehicles and reducing the deleterious effects of spillback \cite{daganzo1998queue} by stopping vehicles further away from entry into a road segment with insufficient capacity. For the larger complex network, DCC with consideration of bottleneck intersections further improves performance by $21.08\%$ and $31.3\%$ respectively compared to the $2$ other approaches tested since it helps congested intersections to clear traffic as soon as possible. In addition to delay, number of stops is also compared. The number of stops of DCC is nearly half of cycle-based adaptive control.

\begin{table}[!htbp]
\centering
  \scalebox{0.6}{
  \begin{tabular}{*{9}{c}}
   \toprule
      \multirow{2}{*}{}& \multicolumn{8}{c}{ Average Delay (second)}  \\
    
   \cmidrule(l){2-9}    & \multicolumn{2}{c}{Benchmark} &  \multicolumn{2}{c}{DCC} &  \multicolumn{2}{c}{DCC w/BC}   &\multicolumn{2}{c}{Cycle-based Adaptive}   \\
   & mean  & std. & mean & std.  & mean & std. & mean &std. \\
    \midrule
 High demand& 212.14 & 361.41& 151.62&  77.13 & 148.62&  62.13 & 230.26 &279.19\\
   Medium demand & 84.22& 61.90& 82.56 &55.84& 78.23 &33.82& 86.46 &61.40 \\
   Low demand& 71.84 &54.25& 72.10 & 49.11 & 70.21 & 42.32 & 73.89 & 56.77 \\

    \bottomrule
  \end{tabular}
  }
   \caption{Average delay under different scenarios.}
  \label{demandtable}
\end{table}

To explore how DCC performs under different demand, we categorize traffic demand into three different groups: low (472 cars/hour), medium (708 cars/hour), and high (1056 cars/hour).  Table~\ref{demandtable} shows DCC to yield an improvement over the baseline approach of about $30\%$ and the cycle-based adaptive control of about $35\%$ for the high traffic demand case. For low and medium traffic, the average delay of three approaches are comparable.

\begin{figure}[!htbp]
\centering
\subfigure[CDF of delay]{\includegraphics[width=40mm, height=35mm]{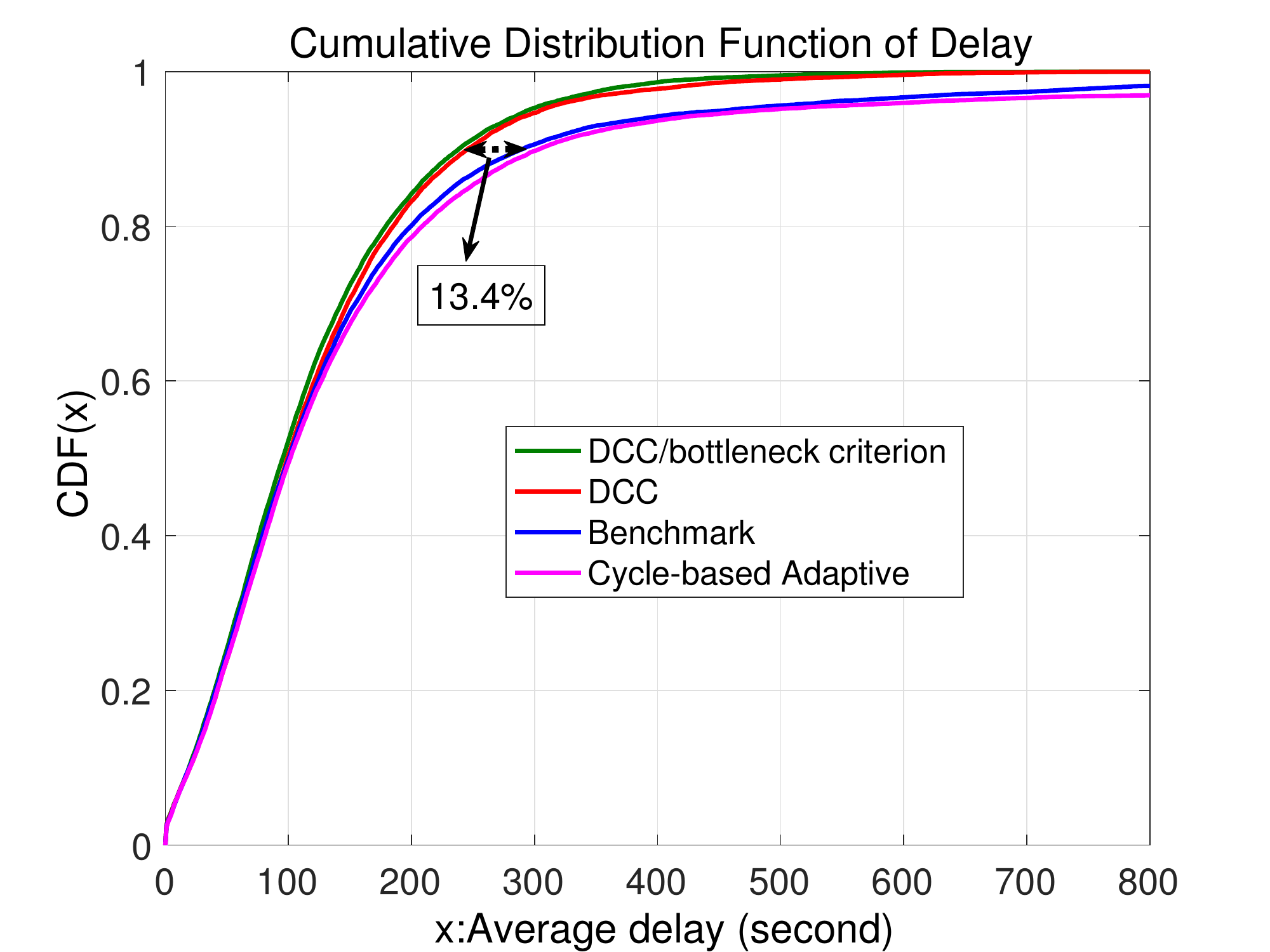}} 
\subfigure[Number of stops]{\includegraphics[width=40mm, height=35mm]{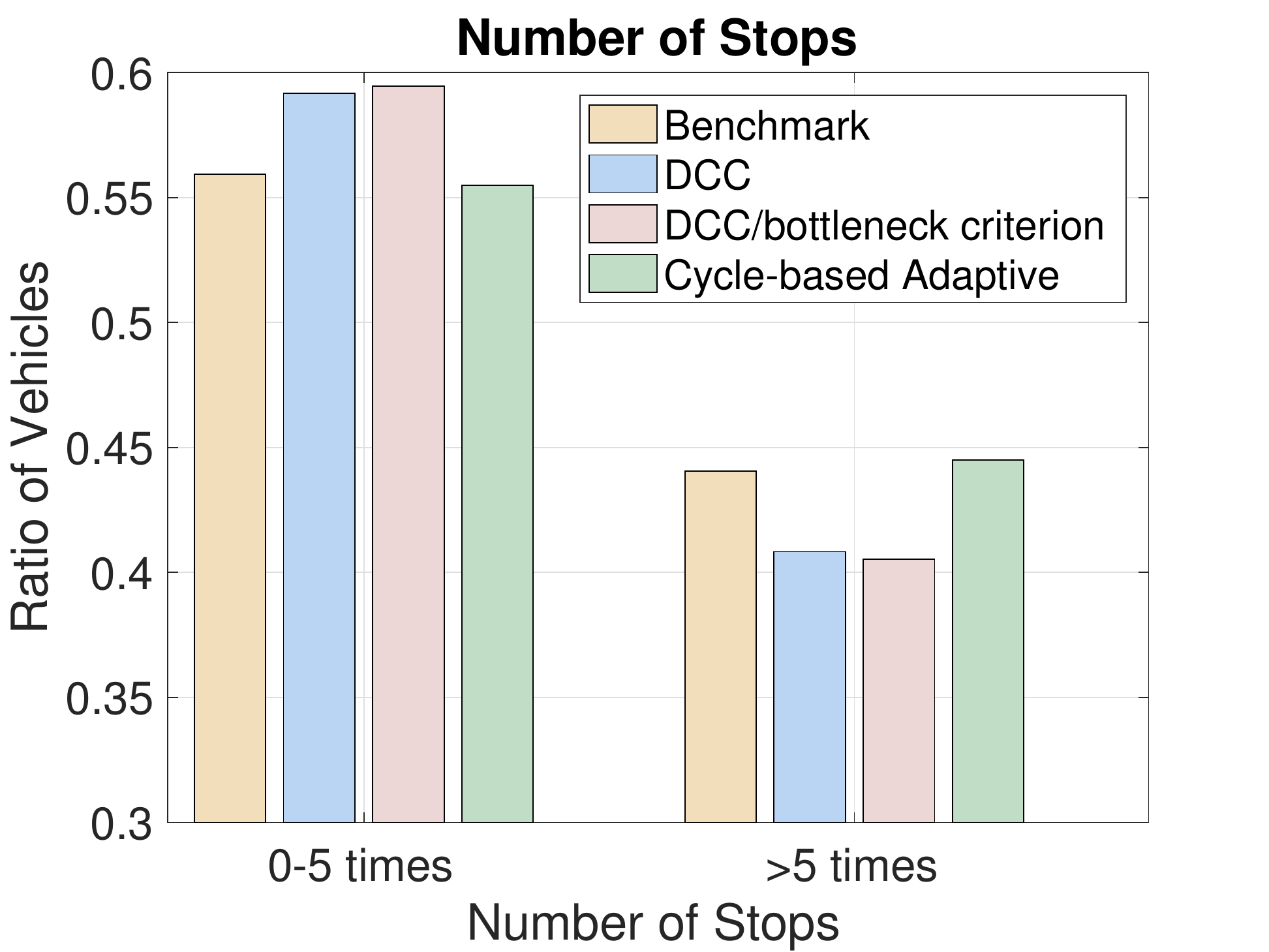}}
\caption{The cumulative distribution function of delay and number of stops.}
\label{cdf}
\end{figure}


Since traffic conditions are dynamically changing, knowing the distribution of delay to vehicles helps us verify the effectiveness of DCC algorithm. As shown in Figure~\ref{cdf} (a), using DCC shifts the cumulative distribution function (CDF) leftward and provides a $13.4\%$ improvement over the baseline approach for $90\%$ of the vehicles. Note also that while DCC reduces average delay by $40s$, the reduction is more than $100s$ for the congested vehicles. In other words, congestion feedback is especially effective for high congestion scenarios. In comparison to adaptive control, DCC provides a $17.6\%$ delay reduction for $90\%$ of the vehicles. As we compare the number of stops among four approaches in the Figure~\ref{cdf} (b), both DCC approaches have fewer vehicles that stop over $5$ times than do the benchmark and adaptive control.


\section{Conclusion}
In this work, we considered the limitations of  prior approaches to schedule-driven traffic control that rely on local optimization without regard to potential consequences due to congestion downstream in the network. A distributed algorithm was proposed to achieve better network-level performance in circumstances of downstream congestion. In this algorithm, agents compute and communicate their expected delay, referred to as congestion feedback, to upstream neighbors in addition to considering the outflow information that is sent downstream. Receiving agents adjust their schedules (and hence their planned outflows) according to this feedback. This delay feedback is computed by interpreting the intersection's generated schedule in an intuitive way and is integrated into the original combinatorial optimization as a form of multi-hop delay. 
Performance was evaluated on two simulation models, a simple two-intersection model and a real-world traffic signal control problem. Results showed that the new bi-directional information exchange model improves average delay overall in comparison to both the baseline schedule-driven traffic control approach and a cycle-based adaptive traffic signal control approach, and that solutions provide substantial gain in highly congested scenarios. Future work will focus on improving the accuracy of feedback based on pricing techniques and negotiation for approaching the optimality of network-wide scheduling. 

\section*{Acknowledgements}This research was funded in part by the University Transportation Center on Technologies for Safe and Efficient Transportation at Carnegie Mellon University and the CMU Robotics Institute.


\bibliographystyle{aaai}
\bibliography{price}  

\end{document}